\documentclass[11pt]{article}

\usepackage[margin=1in]{geometry}
\usepackage[utf8]{inputenc}
\usepackage[T1]{fontenc}
\usepackage{lmodern}
\usepackage{amsmath,amssymb}
\usepackage{booktabs}
\usepackage{graphicx}
\usepackage[dvipsnames]{xcolor}
\usepackage{hyperref}
\usepackage[nameinlink,noabbrev]{cleveref}
\usepackage{enumitem}
\usepackage[numbers,sort&compress]{natbib}
\usepackage{caption}
\usepackage{tabularx}
\usepackage{url}
\usepackage{aliascnt}

\graphicspath{{figures/}}
\hypersetup{
  pdftitle={Monitor Accuracy Isn't the Answer: A Decision-Value Framework for Evaluating AI Agent Restart Monitors},
  pdfauthor={Siddharth Jain and Venkat N Vedam},
  pdfsubject={Decision-value evaluation of progress monitoring and restart control},
  pdfkeywords={AI agents, restart policies, progress monitoring, value of information, software engineering},
  colorlinks=true,
  linkcolor=NavyBlue,
  citecolor=NavyBlue,
  urlcolor=NavyBlue
}

\newtheorem{proposition}{Proposition}[section]
\newaliascnt{lemma}{proposition}
\newtheorem{lemma}[lemma]{Lemma}
\aliascntresetthe{lemma}
\newaliascnt{corollary}{proposition}
\newtheorem{corollary}[corollary]{Corollary}
\aliascntresetthe{corollary}
\newaliascnt{definition}{proposition}
\newtheorem{definition}[definition]{Definition}
\aliascntresetthe{definition}
\crefname{lemma}{lemma}{lemmas}
\Crefname{lemma}{Lemma}{Lemmas}
\crefname{corollary}{corollary}{corollaries}
\Crefname{corollary}{Corollary}{Corollaries}
\crefname{definition}{definition}{definitions}
\Crefname{definition}{Definition}{Definitions}
\newcolumntype{P}[1]{>{\raggedright\arraybackslash}p{#1}}
\newcolumntype{Y}{>{\raggedright\arraybackslash}X}
\newenvironment{proof}{\par\noindent\textit{Proof.}\enspace}{\hfill$\square$\par}

\newcommand{\E}{\mathbb{E}}
\newcommand{\Prob}{\mathbb{P}}
\newcommand{\Fobs}{\mathcal{F}}

\newcommand{\ind}{\mathbb{I}}
\newcommand{\restartlab}{\textsc{RestartLab}}

\newcommand{\blockparagraph}[1]{\vspace{0.8ex}\noindent\textbf{#1}\enspace}
\newcommand{\finding}[1]{%
  \vspace{5pt}\noindent
  \colorbox{gray!9}{%
    \parbox{0.965\linewidth}{\small\strut\textbf{Finding.}\enspace #1\strut}%
  }%
  \vspace{5pt}%
}

\title{Monitor Accuracy Isn't the Answer:\\A Decision-Value Framework for Evaluating\\AI Agent Restart Monitors}

\author{
  Siddharth Jain\\
  Intuit
  \and
  Venkat N Vedam\\
  Intuit
}

\date{}

\begin{document}
\maketitle

\begin{abstract}
Failure monitors for AI agents can be judged by classification accuracy and true-positive and false-positive rates on labeled failures. We show that these metrics alone do not determine what actually matters: whether a restart improves success. Using an exact alarm-to-action identity for an independent clean restart with zero monitoring overhead, we show analytically that a monitor with fixed 90\% accuracy can swing from a five-percentage-point loss to a 35-point gain in success rate, depending solely on the replacement attempt's residual-budget success probability, a quantity accuracy never measures. Building on this, we introduce two techniques for evaluating monitors and progress signals directly in decision terms rather than classification terms. First, a decision-sufficiency criterion, a binary-action specialization of Blackwell's comparison of experiments, identifies exactly which progress representations or context compressions preserve the information needed to choose optimally between the specified actions, independent of how faithfully they reconstruct the trace. Second, a finite-sample, model-free decision audit uses simultaneous binomial intervals over existing agent traces, without additional model calls, to certify a preference for ``restart'' or ``continue'' under stated sampling assumptions, or return an explicit ``abstain'' when evidence is insufficient, instead of reporting an unqualified point estimate. Applied to a study of 194 externally evaluated model executions on 19 selected SymPy issues from SWE-bench Verified, with a retrospective audit of 130 of these traces, this approach exposes what accuracy-based evaluation alone cannot establish: a 2.29-point apparent restart gain traced to a single surviving failed attempt, a 42.29-point loss on a second model, no positive replication of the initial gain on prospective task cohorts, and zero of 103 live comparisons meeting the specified certification bar. Together, these results motivate evaluating agent monitors through recovery opportunity, decision value, and statistical support, rather than classification performance alone.
\end{abstract}

\section{Introduction}

An agent that is unlikely to finish may also be unlikely to succeed after a restart. This distinction matters whenever a harness acts on a failure prediction: abandoning a trajectory removes its remaining chance of success, and the replacement receives only the residual budget. A monitor can therefore classify failures accurately while inducing harmful decisions. Conversely, a modest signal may be useful when it resolves a small number of decisions with large consequences.

The problem has a substantial intellectual history. Dynamic restart policies use observations of a solver's activity to update runtime beliefs~\citep{kautz2002dynamic}; learned schedules address distributions of problem instances~\citep{streeter2007schedules}; Bayesian early stopping derives policies from observed run features~\citep{streeter2019bayes}. Recent agent systems demonstrate useful termination and recovery, including EET~\citep{guo2026eet} and FailFast-RestartSmart~\citep{wang2026failfast}. Our contribution is not the premise that progress can inform stopping. We ask how to make a claimed progress signal's \emph{decision value} and \emph{evidential support} explicit under a finite interaction budget.

Three questions organize the study:
\begin{description}[leftmargin=2.7em,labelindent=0pt,itemsep=2pt,style=nextline]
  \item[\textbf{RQ1}] Which quantities determine whether a failure alarm improves success after a restart?
  \item[\textbf{RQ2}] What information must a progress representation retain, and when do available traces justify a continuation or restart decision?
  \item[\textbf{RQ3}] How do these distinctions explain observed outcomes of restart, horizon extension, context recovery, and failed-branch reuse?
\end{description}

We make the intervention precise. For a live attempt, continuation retains its state and spends the remaining budget; a clean restart restores the original task state and spends the same budget on a fresh attempt. A context reset that retains edits is a separate action, with a separate outcome law. Under this definition, we derive an exact relation between alarm true-positive rate, false-positive rate, continuation success, and restartability. An analytical control keeps failure classification accuracy at $90\%$ while changing restart gain from $-5$ to $+35$ points solely by changing the success probability of the replacement. This is a controlled mathematical example, not a performance claim about an implemented detector.

The information question follows naturally. Progress need not be a monotone scalar. It can be a representation that preserves distinctions relevant to the action. We give a binary-action specialization of classical information comparison~\citep{blackwell1953experiments}: compressing a history loses decision value when the compressed observation combines positive and negative continuation advantages. Exact decision sufficiency is an ideal benchmark; an imperfect representation can still improve utility. We then implement a conservative audit that exposes the statistical support for each action, including an abstention when the evidence does not resolve the sign.

The empirical evidence consists of 194 model-driven executions of mini-swe-agent on 19 selected SymPy issues in SWE-bench Verified~\citep{jimenez2024swebench,yang2024sweagent}. Each attempt begins from the benchmark-provided repository state. The reference patch and evaluator remain outside the sandbox, and external verification after each environment action records the first passing state. Earlier experiments comprise restart replay, prospective task transfer, additional horizon, conversational recovery, and two bounded uses of failed branches. We retain the negative results and distinguish their design status: later collection stages used prospectively frozen protocols, while the decision audit introduced here is retrospective.

The audit sharpens the interpretation of the original results. Across 130 executions with a 30-call horizon, the 16-call Luna gain comes from a single eligible failed trajectory on one task. On the second model, every task with a live attempt at that decision has an empirical advantage for continuation. Across six decision times, all 103 comparisons with live data remain unresolved by the simultaneous certificate. These findings distinguish visible effect size from evidence sufficient for individual deployment decisions. They do not demonstrate that every alternative statistical method must abstain.

\blockparagraph{Contributions.}
\begin{enumerate}[leftmargin=*,itemsep=2pt]
  \item An intervention-specific account of agent progress, including an alarm-to-action value identity and an information-loss criterion for progress representations and context compression.
  \item A dependency-free audit that estimates both action values, respects live eligibility and observed horizons, and returns simultaneous sign certificates or abstention under explicit assumptions.
  \item A reproducible retrospective analysis identifying the support of the restart effect, together with analytical controls that separate detector accuracy, restartability, and sample resolution.
  \item A sequence of externally verified agent experiments showing model-dependent restart effects and the distinction between protocol recovery and verified task completion.
\end{enumerate}

The classical decision identities are attributed and used as foundations. The empirical and software contributions concern how those identities can discipline evaluation of current agent controllers. The study does not train a monitor, directly compare against learned recovery systems, or establish performance outside its selected tasks and runtime configurations.

\section{Related Work}
\label{sec:related}

\subsection{Restarts and inference-time allocation}

Restart policies are well studied for randomized algorithms. Luby, Sinclair, and Zuckerman give a logarithmic expected-runtime comparison with a distribution-specific fixed strategy under their assumptions~\citep{luby1993optimal}. This is not a guarantee that a chosen finite-budget truncation improves success. Dynamic restart policies already condition on solver observations and runtime beliefs~\citep{kautz2002dynamic}. Schedule learning across instances~\citep{streeter2007schedules} and Bayesian stopping from run features~\citep{streeter2019bayes} directly precede our formulation. The no-dominance and observation-ambiguity examples below are elementary specializations; novelty is not claimed for feature-dependent restarting or the need to calibrate runtime distributions.

Our decision-value expression specializes Blackwell's comparison of experiments~\citep{blackwell1953experiments} to two actions. The abstaining audit applies exact binomial intervals~\citep{clopper1934limits} and simultaneous comparison in the spirit of safe policy improvement~\citep{laroche2019safe}. These foundations identify the scope of the contribution: an explicit connection between alarm quality, intervention-specific restartability, and the statistical support of agent decisions, implemented on externally verified traces.

Inference-time scaling allocates additional computation through sampling, voting, or search. Self-consistency samples reasoning paths and aggregates answers~\citep{wang2023selfconsistency}; Tree of Thoughts searches over explicit intermediate candidates~\citep{yao2023tree}; inference scaling studies compare compute-allocation strategies under fixed budgets~\citep{wu2024inference}. Mixture-of-Agents aggregates outputs across model instances and layers~\citep{wang2024mixture}. These methods can exploit independent or weakly correlated trials and informative scoring. Classical ensemble analysis shows why both accuracy and error diversity matter~\citep{dietterich2000ensemble,brown2005diversity}. Our setting removes a central advantage: the controller receives no reliable scalar score before external task completion. We study whether schedule choice or failure-conditioned branching alone can recover that missing information.

Iterative self-improvement methods use generated feedback to revise an answer or policy~\citep{madaan2023selfrefine,shinn2023reflexion}. Other work shows that intrinsic self-correction can fail when the feedback does not add external information~\citep{huang2024selfcorrect}. Our context-rejuvenation and semantic-resolver experiments test a related distinction at the harness level: changing the context or synthesizing failed branches can change behavior without improving the externally verified outcome.

\subsection{Agent progress and externally grounded evaluation}

The closest contemporary system is FailFast-RestartSmart~\citep{wang2026failfast}. It trains a prefix monitor and combines alarms with a fresh rollout that can inspect or apply previous edits. The paper reports both cost savings and recovery gains, providing positive evidence that richer learned signals and interventions can work. EET also demonstrates useful termination from prior execution experience~\citep{guo2026eet}. Our experiments do not reproduce or compare against either system. A timeout is a different controller, and our mandatory workspace retention differs from an optional edit overlay. Our audit complements such systems by requiring the post-alarm action and remaining-budget success to be evaluated jointly with failure prediction.

AgentBoard introduced a progress-rate metric to complement final success across multi-turn environments~\citep{ma2024agentboard}. Such metrics are valuable when an environment exposes meaningful subgoals. Park and Choi show that an agent can report improvement while externally measured state does not improve, and argue for evaluation grounded outside the transcript~\citep{park2026progress}. We take external grounding as a design constraint and ask what remains identifiable when the grounded signal is binary completion rather than a dense progress measure.

Semantic early stopping detects convergence between successive drafts and can reduce tokens in iterative question answering at comparable measured quality~\citep{shrivastava2026semantic}. Geometric analyses represent iterative outputs as trajectories, attractors, or oscillations in an embedding space~\citep{tacheny2026geometric}. These representations describe change in semantic space, but geometric convergence is not automatically decision sufficient for externally verified completion. VLAA-GUI uses visual completeness checks and loop breaking in GUI automation~\citep{han2026vlaa}. These systems have richer task structure than our pass-or-fail coding controller. Their results are consistent with our boundary condition: useful stopping and recovery require a signal calibrated to the relevant external state, not elapsed calls or semantic movement alone.

Long-horizon software agents have also motivated active context management. Context as a Tool lets an agent compress history into structured memory~\citep{liu2025context}, while later work learns or implements agent-compatible context policies~\citep{yi2026adacom,li2026acm}. Our recovery intervention is deliberately simpler. It discards conversational history after a protocol terminal event, retains the exact workspace, and supplies a deterministic state summary. Its null result does not evaluate learned context management. It isolates whether stale conversation alone explains observed failure in one fixed harness.

\subsection{Software engineering agents and reproducible evaluation}

SWE-bench evaluates repository-level issue resolution against executable tests~\citep{jimenez2024swebench}. SWE-agent established the effectiveness of purpose-built agent-computer interfaces for these tasks~\citep{yang2024sweagent}. We use the compact mini-swe-agent implementation as a versioned, open-source harness. We do not compare agent architectures or report a leaderboard estimate. Selected SymPy tasks provide controlled, expensive trajectories with externally testable completion.

Agent evaluation is sensitive to cost, holdout design, and reproducibility~\citep{kapoor2024agents}. Test-suite weakness can also overstate issue resolution~\citep{wang2025solved}. We therefore run the benchmark evaluator outside the agent workspace after each environment action, gate tasks by base-fail and reference-pass checks, audit every successful diff for verifier tampering, and select cutoffs with task-level holdout rather than attempt-level splitting.

\subsection{Position of this work}

Prior work includes both successful stopping systems and formal restart optimization. This study contributes an evaluation and evidence audit for a specified action pair. The distinction is summarized in \cref{tab:positioning}; the negative screens do not invalidate richer systems with different observations, training, or recovery semantics.

\begin{table}[t]
\centering
\caption{The object measured by adjacent lines of work. A task-grounded state can support progress decisions; a trace-only controller must infer it.}
\label{tab:positioning}
\small
\begin{tabular}{@{}P{0.25\linewidth}P{0.25\linewidth}P{0.38\linewidth}@{}}
\toprule
\textbf{Line of work} & \textbf{Observed quantity} & \textbf{Primary question} \\
\midrule
Inference-time search & Candidate scores or votes & Which candidate should receive compute? \\
Dynamic restart theory & Run features and runtime beliefs & Which restart policy minimizes expected runtime? \\
Learned termination and recovery & Prefix alarms and retained edits & Does a trained controller save cost or improve resolution? \\
Semantic stopping & Change between successive outputs & Has an iterative answer stabilized? \\
Trajectory evaluation & Task-specific subgoal state & How much of the task is complete? \\
Context management & Structured or compressed history & What information should the agent retain? \\
This work & Action values, live eligibility, sampling uncertainty & Does the intervention have decision value, and is its sign supported? \\
\bottomrule
\end{tabular}
\end{table}

\section{Black-Box Progress as a Decision Problem}
\label{sec:theory}

\subsection{Setup}

Fix a task, agent, model, provider policy, prompt, and tool interface. One attempt induces latent states $X_0,X_1,\ldots$ and observations $O_1,O_2,\ldots$. This is a partially observed control process~\citep{kaelbling1998pomdp}. The controller's information after call $t$ is the filtration
\begin{equation}
  \Fobs_t = \sigma(O_1,\ldots,O_t),
\end{equation}
where observations may include actions, tool outputs, workspace summaries, external verifier results, and explicit protocol failures. Let
\begin{equation}
  T = \inf\{t \geq 1 : V(X_t)=1\}
\end{equation}
be the first call whose resulting state passes the external verifier, with $T=\infty$ if no such state is reached. The belief state $b_t(dx)=\Prob(X_t\in dx\mid\Fobs_t)$ is sufficient for optimal control when the generative process is known. A black-box harness does not know this process and must approximate decision value from its observations. A budget policy $\pi$ maps observable histories to continuation, restart, recovery, branch, or termination actions under total call budget $B$.

We use \emph{semantic progress} for a latent property of state that predicts future verified completion. This term does not presume that progress is monotone or one dimensional. Let $Q_a(h_t,b)$ denote expected utility when action $a$ is selected after history $h_t$ with $b$ calls remaining. For the binary continue-or-restart problem, define the continuation advantage
\begin{equation}
  \Delta(h_t,b)=Q_{\mathrm{cont}}(h_t,b)-Q_{\mathrm{restart}}(h_t,b).
  \label{eq:advantage}
\end{equation}
The utility may combine externally verified success, calls, money, and latency, but it must be declared before evaluating the policy.

\begin{definition}[Decision-sufficient progress representation]
\label{def:sufficient}
An observable representation $G_t=g(O_{1:t})$ is decision sufficient for a policy class and runtime population if the optimal action is measurable with respect to $(G_t,b)$. For binary control, this requires the sign of \cref{eq:advantage} to be determined by $(G_t,b)$ almost surely.
\end{definition}

This definition describes lossless action selection relative to the richer history. A representation can be useful without attaining exact sufficiency. Embedding displacement, code churn, test counts, and fluent self-reports may describe a trajectory while discarding a distinction that changes the preferred action. Two histories may be close in a semantic space yet have different continuation advantages, or look dissimilar while implying the same action.

\begin{proposition}[Progress error is advantage-weighted decision error]
\label{prop:weighted-regret}
Let $d(h_t,b)\in\{\mathrm{cont},\mathrm{restart}\}$ be any stopping rule and let $d^*(h_t,b)$ select the action with larger conditional value. Its expected regret within the binary action class is
\begin{equation}
  \mathcal{R}(d)
  =\E\!\left[
    |\Delta(H_t,b)|
    \ind\{d(H_t,b)\neq d^*(H_t,b)\}
  \right].
  \label{eq:weighted-regret}
\end{equation}
\end{proposition}

\begin{proof}
Conditioned on a history, choosing the inferior action loses exactly the absolute difference between the two action values; choosing the superior action loses zero. Taking expectation gives \cref{eq:weighted-regret}.
\end{proof}

The identity shows why label accuracy is not the right primary metric for a progress detector. Errors matter in proportion to the value difference between continuing and restarting. A detector can achieve high accuracy on histories where the actions are nearly tied and still lose most of the attainable utility on the few histories where $|\Delta|$ is large.

\begin{figure}[t]
\centering
\includegraphics[width=\linewidth]{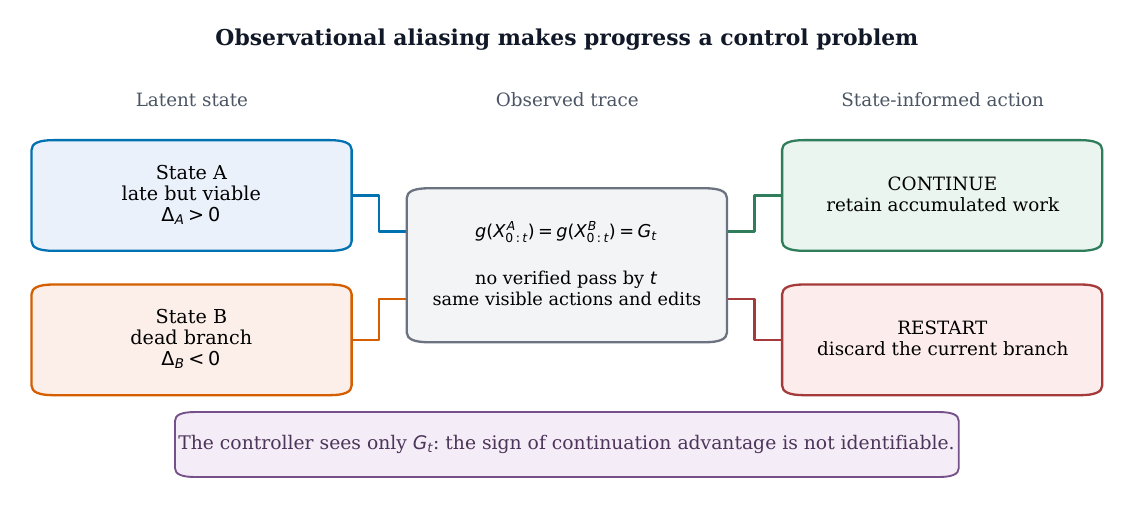}
\caption{Observational aliasing in black-box control. Distinct latent states can induce the same observable history while requiring opposite actions. A useful progress sensor must separate histories according to continuation advantage, not surface activity.}
\label{fig:observability}
\end{figure}

\subsection{Fixed restarts depend on the completion scale}

Suppose fresh attempts are independent and identically distributed, restart overhead is zero, and a live attempt is truncated at cutoff $c$. Write $F(t)=\Prob(T\leq t)$ and decompose the total budget as $B=qc+r$, where $q=\lfloor B/c\rfloor$ and $0\leq r<c$. If unused budget is assigned to a final residual episode, the success probability is
\begin{equation}
  P_c(B)=1-\left(1-F(c)\right)^q\left(1-F(r)\right),
  \label{eq:fixed-restart}
\end{equation}
with $F(0)=0$. This expression assumes that noncompletion consumes the assigned episode budget. An explicit terminal event can end an episode earlier and release additional budget, which is handled by trace-level replay in the experiments.
An uninterrupted attempt succeeds with probability $F(B)$.

The discrete completion hazard makes the missing scale explicit. Let
\begin{equation}
  \lambda_u=\Prob(T=u\mid T\geq u).
\end{equation}
With $r$ calls remaining and no observation beyond survival through $t$, continuation succeeds with probability
\begin{equation}
  C(t,r)=1-\prod_{u=t+1}^{t+r}(1-\lambda_u),
  \label{eq:residual}
\end{equation}
whereas a fresh attempt succeeds with probability $R(r)=F(r)$. Under equal success utility and zero reset cost, restart is preferred exactly when $R(r)>C(t,r)$. Thus elapsed time helps only through assumptions about the hazard. A decreasing hazard favors restarts; delayed completion mass can favor continuation. With trace features $G_t$, the relevant object is the feature-conditional residual law, not the marginal curve alone.

\begin{proposition}[No uniformly dominant proper cutoff]
\label{prop:no-dominance}
For any $B\geq2$ and any fixed cutoff $1\leq c<B$, there exists a completion-time distribution for which the uninterrupted policy strictly dominates cutoff $c$. There also exists a completion-time distribution for which cutoff $c$ strictly dominates the uninterrupted policy whenever $\lfloor B/c\rfloor\geq2$.
\end{proposition}

\begin{proof}
For the first claim, place all probability mass at $T=c+1\leq B$. Then $F(c)=0$ and $F(B)=1$. For the second, let $T=c$ with probability $p\in(0,1)$ and let $T=\infty$ otherwise. Then $F(c)=F(B)=p$, $F(r)=0$, and \cref{eq:fixed-restart} equals $1-(1-p)^q>p$ for $q=\lfloor B/c\rfloor\geq2$.
\end{proof}

The proposition is elementary but constraining. Any positive universal-restart claim must introduce assumptions on the completion distribution, a prior over task-agent tuples, or an observation that predicts the residual completion time. Schedule shape cannot supply this information.

\begin{corollary}[Finite-scale sensitivity]
Consider any finite restart schedule whose largest episode has length $s<B$. There is a distribution supported on completion times in $(s,B]$ for which the schedule has zero success and an uninterrupted attempt succeeds with positive probability.
\end{corollary}

This corollary does not contradict the classical Luby guarantee~\citep{luby1993optimal}. It identifies the finite-budget role of the base scale. A universal ordering of episode lengths is not a universal calibration of model calls.

\subsection{The continuation decision is not identified by survival}

At call $c$, suppose the only task-grounded event is that no verified state has appeared. The remaining budget is enough for either one more call on the current attempt or one call on a fresh attempt.

\begin{proposition}[Observational ambiguity]
\label{prop:ambiguity}
There are two latent completion regimes that induce the same observable no-success history through call $c$ but require opposite optimal decisions. In one regime continuation succeeds with probability one and restart fails. In the other, conditional continuation fails while restart has positive success probability. Consequently, no rule measurable with respect to that history is optimal in both regimes.
\end{proposition}

\begin{proof}
In the late regime, let the current attempt complete deterministically at $c+1$ and let a fresh attempt require more than one call. Continuation is uniquely optimal. In the fast-or-never regime, an attempt completes on its first call with probability $p>0$ and never completes otherwise. Conditional on survival through $c$, the current attempt is in the never component, so continuation fails and a fresh attempt succeeds with probability $p$. Both regimes can emit the same observation sequence before external completion. A deterministic rule chooses the wrong action in one regime. A randomized rule cannot attain the regime-specific optimum in both.
\end{proof}

The result is about identification, not metaphysical opacity. A task-specific sensor may distinguish the regimes. Absent such a sensor or a calibrated prior, a controller that restarts because the trace \emph{looks stagnant} imports an untested mapping from appearance to residual value.

\begin{corollary}[Regret under observational aliasing]
\label{cor:regret}
In the construction of \cref{prop:ambiguity}, suppose a fresh attempt in the fast-or-never regime succeeds with probability $p\in(0,1)$. For any randomized policy that observes only the aliased history, the worst-case regret over the two regimes is at least $p/(1+p)$. The bound approaches $1/2$ as $p$ approaches one.
\end{corollary}

\begin{proof}
Let the policy continue with probability $\alpha$. Its regret is $1-\alpha$ in the late regime and $p\alpha$ in the fast-or-never regime. Minimizing $\max\{1-\alpha,p\alpha\}$ gives $\alpha=1/(1+p)$ and regret $p/(1+p)$.
\end{proof}

The corollary gives an operational meaning to unobservability. It is not merely that a trace classifier may be inaccurate. Without an identifying assumption, error remains even with unlimited computation applied to the same observable history.

An additional failure-selection example and the standard correlated-ensemble variance identity appear in Appendix~\ref{app:ensemble}. They motivate the branch-reuse screens but are not empirical estimates of our patch population.

\section{From Failure Prediction to Restart Decisions}
\label{sec:decision-value}

The useful question is whether a signal changes expected utility under a specified intervention. This section connects the classical value of information~\citep{blackwell1953experiments} to a finite-budget audit that can be computed without model access. The identities specialize existing decision theory; the contribution is their operationalization for agent termination and recovery, together with an explicit account of what the available traces can identify.

\subsection{Failure prediction must be combined with restartability}

Fix one task and runtime, a live decision at call $t$, and residual budget $b=B-t$. Let $Y_C$ be the potential binary outcome of uninterrupted continuation to the deadline. Let $q=\Prob(Y_C=1)$, and let $p$ be the success probability of one independent clean restart with the same residual budget. We call $p$ the \emph{restartability} of this decision. An alarm $A=1$ triggers the restart. Define its failure true-positive rate $\tau=\Prob(A=1\mid Y_C=0)$ and false-positive rate $\phi=\Prob(A=1\mid Y_C=1)$.

\begin{proposition}[Alarm-to-action value identity]
\label{prop:alarm-value}
Assume the fresh outcome is independent of $(A,Y_C)$ conditional on the fixed task and runtime, the restart restores the initial state, and the alarm consumes no budget. The success change relative to continuing every live attempt is
\begin{equation}
  G(q,p,\tau,\phi)=(1-q)\tau p-q\phi(1-p).
  \label{eq:alarm-value}
\end{equation}
If $(1-q)\tau+q\phi>0$, a positive gain requires
\begin{equation}
  p>\frac{q\phi}{(1-q)\tau+q\phi}.
  \label{eq:restartability-threshold}
\end{equation}
\end{proposition}

\begin{proof}
On correctly alarmed failures, the restart changes success from zero to $p$, contributing $(1-q)\tau p$. On falsely alarmed successes, it changes success from one to $p$, contributing $-q\phi(1-p)$. Unalarmed outcomes are unchanged. Rearrangement gives the threshold.
\end{proof}

For $q=1/2$, $\tau=0.9$, and $\phi=0.1$, the detector's classification accuracy is $90\%$ for every $p$, while $G=0.5p-0.05$. Thus identical alarm quality gives a five-point loss at $p=0$, no change at $p=0.1$, and a 35-point gain at $p=0.8$. Even perfect failure prediction produces no additional successes when $p=0$, although early termination can still save cost. These are analytical controls, not measured monitor performance. They motivate joint reporting of alarm discrimination, residual-budget restartability, and realized utility. A global accuracy or AUROC alone omits the second quantity.

An edit-preserving restart changes the intervention and need not satisfy the independence assumption. For a general alarm, the exact identity is $G=\E[A(Y_R-Y_C)]$. Computing it requires outcomes for the specified recovery action. Mixing task populations can also induce dependence, even when fresh draws are independent within each task. We therefore estimate and interpret action values within task and runtime before aggregation.

\begin{corollary}[Opportunity ceiling for a failure alarm]
\label{cor:opportunity}
Under the assumptions of \cref{prop:alarm-value}, any alarm satisfies
\begin{equation}
  G\leq (1-q)p.
  \label{eq:opportunity-ceiling}
\end{equation}
Equality is attained by a perfect alarm, $A=1-Y_C$. The additional value of this clairvoyant alarm over the better constant action is
\begin{equation}
  (1-q)p-(p-q)_+=\min\{(1-q)p,\ q(1-p)\}.
\end{equation}
\end{corollary}

\begin{proof}
In \cref{eq:alarm-value}, use $\tau\leq1$ and $\phi\geq0$. Subtract the better constant-action gain $(p-q)_+$ and distinguish $p\geq q$ from $p<q$.
\end{proof}

The ceiling is a value-of-perfect-information benchmark for one specified action pair. It is neither a realizable detector nor an upper bound for state-preserving recovery, multiple restarts, or another model. A plug-in estimate of this ceiling can diagnose an evaluation set with little measured opportunity; it does not establish that the true opportunity is small without sampling uncertainty.

\subsection{The decision value of a richer progress signal}

Let $Z$ be a richer observation and let $D(Z)=Q_R(Z)-Q_C(Z)=-\Delta(Z)$ denote restart advantage. A coarser representation is $G=g(Z)$. Fix the residual budget and population of live decisions. The best policy using $G$ has value
\begin{equation}
  V(G)=\E[Q_C(Z)]+\E\left[(\E[D(Z)\mid G])_+\right],
\end{equation}
where $(x)_+=\max(x,0)$.

\begin{proposition}[Value lost by compressing a progress signal]
\label{prop:information-value}
The loss from replacing $Z$ by $G=g(Z)$ is
\begin{equation}
  V(Z)-V(G)=\frac12\left\{\E|D(Z)|-\E|\E[D(Z)\mid G]|\right\}\geq0.
  \label{eq:information-value}
\end{equation}
It vanishes when each coarse observation contains advantages of only one sign, up to zero-probability events. In particular, preserving the optimal action can preserve decision value without reconstructing the original trace.
\end{proposition}

\begin{proof}
Using $x_+=(x+|x|)/2$ in both values cancels the common mean advantage. Conditional Jensen's inequality proves nonnegativity. Equality in the absolute-value inequality holds exactly when positive and negative advantages do not coexist within a conditional observation class.
\end{proof}

This gives a task-specific criterion for context compression: preserve distinctions that change the preferred action. Token retention, embedding fidelity, and reconstruction quality do not by themselves measure this loss. The identity is a binary-action specialization of classical information comparison, not a new universal information theorem. It also permits imperfect signals to have positive value; exact decision sufficiency is an ideal reference, not a prerequisite for practical improvement. If acquiring $Z$ reduces expected utility by $\kappa$ while otherwise preserving the action-value model, its net benefit is $V(Z)-V(G)-\kappa$. When acquisition consumes calls or changes the workspace, action values must instead be evaluated at the resulting budget and state.

\subsection{A finite-sample decision certificate}
\label{sec:certificate}

Let $E$ denote explicit protocol termination. At a fixed $t$, a trace is eligible only if it is still live:
\begin{equation}
  L_t=\{T>t,\ E>t\}.
\end{equation}
For task $j$, the audited comparison is uninterrupted continuation through $B$ versus one clean restart with $B-t$ calls:
\begin{equation}
  q_{j,t}=\Prob(T\leq B\mid L_t,j),\qquad
  p_{j,t}=\Prob(T\leq B-t\mid j).
  \label{eq:audit-target}
\end{equation}
The first quantity is estimated from eligible attempts; the second from all fresh attempts on that task. A later protocol terminal counts as unsuccessful continuation. An earlier terminal is outside the live-decision population. Neither arm extrapolates beyond the observed horizon.

For a fixed family of $M$ task-time comparisons, form two-sided Clopper-Pearson intervals~\citep{clopper1934limits} for every $q_{j,t}$ and $p_{j,t}$, each with error probability $\alpha/(2M)$. Write them as $[L_C,U_C]$ and $[L_R,U_R]$. An empty eligible sample receives $[0,1]$. The decision audit returns
\begin{equation}
  \mathcal D_{j,t}=
  \begin{cases}
    \mathrm{restart}, & L_R>U_C,\\
    \mathrm{continue}, & U_R<L_C,\\
    \mathrm{abstain}, & \text{otherwise}.
  \end{cases}
  \label{eq:certificate}
\end{equation}

\begin{proposition}[Simultaneous sign certification]
\label{prop:certificate}
For IID attempts within each fixed task and runtime, complete observation through $B$, and a comparison family fixed independently of the evaluated outcomes, every nonabstaining decision in \cref{eq:certificate} has the correct sign with probability at least $1-\alpha$.
\end{proposition}

\begin{proof}
Each binomial interval has its stated marginal coverage, including the continuation interval conditional on the number of eligible attempts. A union bound over the $2M$ intervals gives simultaneous coverage. On that event, $p_{j,t}-q_{j,t}\in[L_R-U_C,U_R-L_C]$, so either strict separation certifies the sign. Independence between the two estimated arms is unnecessary for the union bound; the samples may overlap.
\end{proof}

This certificate implements a simple form of conservative policy improvement~\citep{laroche2019safe}. It is a sufficient evidence rule, not an optimal test or an impossibility bound for all estimators. Abstention means that this procedure does not resolve the comparison. Falling back to continuation after abstention does not certify continuation as optimal. Deployment also requires the calibration population to match the future runtime; the guarantee does not cover arbitrary distribution shift or uncorrected repeated checking.

\blockparagraph{Evidence cost at zero success.}
For $n$ observations with zero successes, the upper endpoint under this interval allocation is
\begin{equation}
  U_0(n,M)=1-\left(\frac{\alpha}{4M}\right)^{1/n}.
  \label{eq:zero-bound}
\end{equation}
To obtain $U_0\leq\epsilon$ requires
\begin{equation}
  n\geq\left\lceil\frac{\log(\alpha/(4M))}{\log(1-\epsilon)}\right\rceil.
  \label{eq:resolution-cost}
\end{equation}
These counts describe this specified binomial procedure. They quantify the evidence needed to bound one success probability, not sufficient sample sizes to establish an intervention effect. Conditional continuation additionally requires that these observations reach the decision alive.

\begin{figure}[t]
\centering
\includegraphics[width=\linewidth]{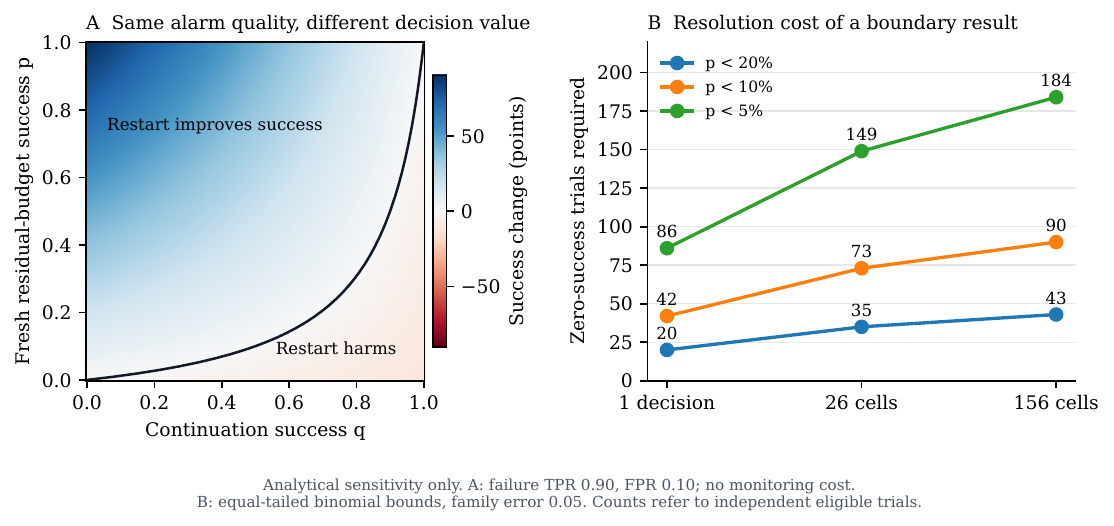}
\caption{Analytical controls for monitor evaluation. Panel A keeps the failure detector's true-positive and false-positive rates fixed while varying continuation success and restartability. The black curve is the zero-gain boundary. Panel B gives the number of zero-success observations required by the specified simultaneous interval procedure to bound a success probability. Neither panel reports newly collected agent performance.}
\label{fig:monitor-frontier}
\end{figure}

\section{Experimental Method}
\label{sec:method}

\subsection{Design principle}

We use a sequential falsification design. Each experiment was motivated by the preceding result, but its task rule, intervention, outcome, and stopping rule were frozen before new outcomes were observed. Earlier data are used only where stated. This structure does not make an adaptive research program equivalent to one preregistered confirmatory study. It does prevent within-stage changes to the tested intervention after seeing its result.

\begin{table}[t]
\centering
\caption{Experimental ladder. Model attempts count actual agent or resolver executions. Patch algebra used existing failed workspaces and no model calls.}
\label{tab:ladder}
\small
\begin{tabular}{@{}p{0.19\linewidth}p{0.36\linewidth}r p{0.27\linewidth}@{}}
\toprule
\textbf{Stage} & \textbf{Question} & \textbf{Attempts} & \textbf{Frozen comparison} \\
\midrule
Initial restart & Does a short cutoff beat one 30-call attempt? & 70 & Cutoffs and universal schedules \\
Prospective transfer & Does the selected cutoff transfer to unseen tasks? & 60 & Fixed 16 versus fixed 30 \\
Horizon extension & Is completion mass hidden after call 30? & 20 & One 60 versus repeated 30 \\
Context recovery & Does fresh context recover protocol collapse? & 40 & Baseline versus rejuvenation \\
Patch algebra & Do clean unions of failed patches solve? & 0 & 19 frozen composites \\
Semantic resolver & Can one agent synthesize failed hypotheses? & 4 & Four bounded portfolio runs \\
\midrule
Total & 19 unique tasks, two models & 194 & Plus 19 verifier-only candidates \\
\bottomrule
\end{tabular}
\end{table}

\subsection{Benchmark and runtime tuple}

We use 19 issues from the SymPy repository in SWE-bench Verified~\citep{jimenez2024swebench}. Restricting to one repository removes language and build-system variation, while retaining issue-level differences in symbolic mathematics, printing, assumptions, matrices, and related subsystems. It also limits external validity, which we treat explicitly in \cref{sec:limitations}.

An experimental runtime is identified by benchmark task, model endpoint, agent release, provider route, prompt configuration, and sampling policy. We use mini-swe-agent 2.4.6~\citep{minisweagent2026}. Attempts use temperature $0.7$, restore the benchmark-specified repository state before each run, execute under WSL2 Bubblewrap isolation without network access, and expose a workspace visible only to the attempt. The initial comparison uses \texttt{openai/gpt-5.6-luna} and \texttt{deepseek/deepseek-v4-flash-0731}. All subsequent model-driven experiments use the former.

The primary budget unit is a model call. Dollar cost and wall time are secondary operational measures. A 30-call attempt need not emit 30 valid environment actions, and an explicit format failure can end it early. We retain the distinction between nominal horizon, first verified completion, and protocol-terminal call.

Formally, let $E$ be the first explicit protocol-terminal call and let $H$ be the configured horizon. The trace records
\begin{equation}
  U=\min\{T,E,H\},
  \qquad
  J\in\{\text{verified},\text{protocol terminal},\text{right censored}\}.
  \label{eq:trace-outcome}
\end{equation}
The event label $J$ matters for replay. A protocol terminal can release calls to a new episode; a right-censored live trajectory cannot be treated as known failure beyond $H$.

\begin{table}[t]
\centering
\caption{Runtime controls shared across model-driven experiments unless a stage states otherwise.}
\label{tab:runtime}
\small
\begin{tabularx}{\linewidth}{@{}p{0.28\linewidth}X@{}}
\toprule
\textbf{Component} & \textbf{Frozen value} \\
\midrule
Benchmark & SWE-bench Verified, selected \texttt{sympy/sympy} tasks \\
Agent & mini-swe-agent 2.4.6 \\
Models & \texttt{openai/gpt-5.6-luna};\newline historical comparator \texttt{deepseek/deepseek-v4-flash-0731} \\
Sampling & Temperature 0.7; separate provider draws \\
Primary budget & Model calls, with stage horizons of 12, 30, or 60 \\
Verification & External fail-to-pass evaluator after every environment action \\
Isolation & Benchmark-specified task state, network-disabled Bubblewrap, evaluator outside workspace \\
Collection dates & August 31 and September 6, 2026 \\
\bottomrule
\end{tabularx}
\end{table}

\subsection{External verification and completion time}

Each selected task passes two environment gates before collection: the benchmark-provided starting state fails the task evaluator and the reference patch passes it. The evaluator and test patch remain outside the agent sandbox. After every environment action, the controller transfers the candidate diff into a clean evaluator and runs the hidden fail-to-pass tests. Completion time is the first passing state. This prevents a later destructive action or an inaccurate self-declaration from changing the recorded time.

Every successful model attempt was audited for changes to the verifier, test runner, or hidden tests. None modified those assets. Successful patches changed implementation files, with occasional scratch reproductions. The analysis never exposes hidden verifier output, the reference patch, or completion labels to the model.

\subsection{Restart replay}

The initial cohort contains five nominally independent 30-call attempts for each of seven tasks and two models, totaling 70 observations. We evaluate a fixed cutoff $c$ under total budget 30 by empirical resampling of fresh attempts within task. Write $30=qc+r$, with $q=\lfloor30/c\rfloor$ and $0\leq r<c$. In the absence of an earlier protocol-terminal event, the plug-in probability is
\begin{equation}
  \widehat P_{j,c}(30)=1-\left(1-\widehat F_j(c)\right)^q
  \left(1-\widehat F_j(r)\right).
\end{equation}
The exact replay additionally releases unused calls when a recorded protocol-terminal event occurs before the assigned cutoff. Macro success averages the resulting probability across tasks. Cutoffs range from 4 to 30 in steps of two. The full-horizon cutoff is the uninterrupted baseline. For geometric and Luby schedules, the implementation uses Monte Carlo replay with 50,000 episode sequences per task.

Cutoff selection is evaluated in sample and by leave-one-task-out selection. Each fold selects the best cutoff on the other tasks, breaking ties toward the longer cutoff, then evaluates it on the held-out task. No attempt-level split is used because attempts from one issue share task difficulty.

\subsection{Prospective task cohorts}

The first new cohort selects the seven highest-ID unseen SymPy tasks that pass the environment gates. This mechanical cohort produced 35 attempts. Because all were solved early, a second difficulty-enriched cohort targeted unseen issues labeled at least one hour, then filled rejected positions from the highest-ID unseen issues labeled 15 minutes to one hour. Five tasks passed the gates, yielding 25 attempts. The transferred cutoff of 16 calls and the 30-call baseline were fixed before either outcome set was observed.

The 60-call extension contains five fresh attempts on each of four stress tasks that were $0/5$ at horizon 30. Its primary outcome is any first completion in calls 31 through 60. A secondary analysis separates right-censoring from explicit repeated-format termination.

\subsection{Context-rejuvenation intervention}

Four tasks that exhibited repeated-format termination receive five fresh baseline and five treatment attempts each. Both conditions have a shared 60-call cap. After three consecutive invalid tool responses, treatment preserves the exact repository state, discards the conversation, instantiates a fresh model client, and supplies the original issue plus deterministic \texttt{git status} and diff-stat summaries. It receives no extra calls, hidden feedback, progress score, or model checkpoint. Baseline follows the unmodified harness behavior.

The prospective expansion rule required at least a five-point success gain or two recovered successes, with no more than a $20\%$ mean-cost increase. The rule was frozen before collection.

\subsection{Failed-branch reuse}

The context experiment yielded 38 externally failed, nonempty, production-code-only terminal patches after excluding one verified patch and one empty patch. We deduplicated them by content, enumerated unordered within-task pairs, and performed clean three-way composition against the common benchmark starting state. Candidate generation used no model, verifier output, or gold patch. The ranking was frozen before verification and allowed at most eight candidates per task. Nineteen candidates were evaluated.

The semantic-resolver screen then supplied one clean 12-call Luna attempt per task with the issue and every unique failed patch for that task, each explicitly labeled as a failed hypothesis. The resolver could inspect and edit a clean workspace. The four contexts were frozen before execution. The protocol allowed no retries, and a second stage would open only after at least one verified success.

\subsection{Retrospective decision-value audit}
\label{sec:audit-method}

We reanalyze the 70 initial and 60 prospective 30-call attempts, without additional model calls. These 130 observations form 26 task-model cells with five attempts each. The audit evaluates $t\in\{4,8,12,16,20,24\}$ at total budget $B=30$, giving a family of $M=156$ comparisons. It uses the live eligibility rule and binomial intervals in \cref{sec:certificate}, with $\alpha=0.05$. Empty live populations are reported as not reached, not as successful abstentions. All attempt identities are deduplicated before analysis.

The decision is a single clean restart using the residual budget, not the repeated-cutoff policy used in the earlier replay. We report conditional action values within each task, then weight the change by the observed fraction reaching the decision and average tasks equally. A task oracle chooses the better estimated action separately for each task; a perfect failure alarm additionally knows which live continuations will fail, as in \cref{cor:opportunity}. Both are retrospective plug-in references, not deployed policies or validated opportunity bounds. In particular, zero observed restart success does not establish zero true restartability.

This audit was designed after inspection of the collected experiments. Its comparison family and analysis are fully specified for reproduction, but the audit is exploratory, not preregistered or a prospective test of the coverage guarantee. It supplies no paired state-conditioned intervention outcomes and does not evaluate a learned prefix monitor.

\subsection{Uncertainty and reporting}

For policy contrasts, percentile bootstrap samples tasks with replacement and then attempts within sampled task. Unless specified, intervals use 10,000 replicates and report the empirical $2.5\%$ and $97.5\%$ quantiles. We report bootstrap sign probabilities as descriptive quantities, not posterior probabilities or frequentist $p$ values. Boundary outcomes are accompanied by an unclustered Newcombe interval as a sensitivity check where useful. All success estimates are absolute percentage-point differences.

For task $j$ and policy $\pi$, let $\theta_{j,\pi}(B)$ be the probability of external verification within total budget $B$. The primary estimand is the task-macro policy value
\begin{equation}
  \Theta_{\pi}(B)=\frac{1}{J}\sum_{j=1}^{J}\theta_{j,\pi}(B),
\end{equation}
and a policy contrast is $\Delta_{\pi,\pi_0}=\Theta_{\pi}(B)-\Theta_{\pi_0}(B)$. Equal task weighting prevents an issue with more retained attempts from dominating the result. The hierarchical bootstrap targets variability over both sampled tasks and provider draws within task. With only one repository and purposefully selected cohorts, it should not be read as a population interval over all software engineering tasks.

Unsuccessful attempts at a horizon are right-censored: they establish $T>c$, not $T=\infty$~\citep{kalbfleisch2002failure}. Offline replay never evaluates a cutoff beyond the observed horizon. Explicit protocol termination is a competing event, recorded at its observed call, and may release the residual budget to a new attempt. Treating these terminal events as ordinary censoring would overcharge the episode and obscure the distinction between time-based truncation and fault recovery.

\section{Results}
\label{sec:results}

\Cref{tab:main-results} states the principal contrasts before the stage-by-stage analysis. The paper gives priority to frozen prospective comparisons over in-sample maxima.

\begin{table}[t]
\centering
\caption{Main empirical results. Intervals resample tasks and attempts unless noted.}
\label{tab:main-results}
\small
\begin{tabularx}{\linewidth}{@{}P{0.28\linewidth}P{0.26\linewidth}Y@{}}
\toprule
\textbf{Comparison} & \textbf{Observed result} & \textbf{Interpretation} \\
\midrule
Fixed 16 versus 30, Luna training tasks & $+2.29$ points $[0.00,6.86]$ & Weak in-sample gain \\
Luna-selected fixed 16 on DeepSeek & $-42.29$ points $[-69.14,-17.71]$ & Selected scale does not transfer across models \\
Fixed 16 versus 30, all 19 Luna tasks & $-0.21$ points $[-4.21,2.53]$ & Pooled estimate includes initial and prospective tasks \\
First completion after call 30 & 0/20; one-sided upper bound $13.9\%$ & No observed useful late tail; rare mass remains possible \\
Fresh context versus baseline & 0/20 versus 1/20; mean cost $+56.7\%$ & Protocol label improves, task outcome does not \\
Failed-branch reuse & 0/19 composites; 0/4 resolvers & Textual diversity does not supply verified closure \\
\bottomrule
\end{tabularx}
\end{table}

\subsection{A restart scale does not transfer across models}

The initial 70 attempts already contain two different completion-time regimes under the same tasks and harness. Luna solves 29 of 35 attempts, mostly early. DeepSeek solves 24 of 35, but many completions occur after call 16. Per-task outcomes appear in \cref{tab:initial-outcomes}; the empirical completion distributions and fixed-cutoff policy values appear in \cref{fig:cutoffs}.

\begin{table}[t]
\centering
\caption{Initial externally verified outcomes. Call statistics are conditional on successful attempts.}
\label{tab:initial-outcomes}
\small
\begin{tabular}{@{}lrrlr l@{}}
\toprule
& \multicolumn{2}{c}{\textbf{Luna}} & & \multicolumn{2}{c}{\textbf{DeepSeek}} \\
\cmidrule{2-3}\cmidrule{5-6}
\textbf{Task} & \textbf{Pass} & \textbf{Calls: min/mean/max} && \textbf{Pass} & \textbf{Calls: min/mean/max} \\
\midrule
23413 & 4/5 & 5/7.0/9 && 4/5 & 21/23.8/27 \\
23824 & 5/5 & 4/5.2/6 && 5/5 & 11/12.8/15 \\
23950 & 0/5 & none && 1/5 & 29/29.0/29 \\
24066 & 5/5 & 5/7.0/11 && 3/5 & 20/25.0/29 \\
24213 & 5/5 & 4/4.6/5 && 5/5 & 13/19.8/28 \\
24443 & 5/5 & 8/12.0/16 && 4/5 & 19/23.0/30 \\
24661 & 5/5 & 6/7.4/9 && 2/5 & 15/22.5/30 \\
\midrule
Total & 29/35 &  && 24/35 & \\
\bottomrule
\end{tabular}
\end{table}

\begin{figure}[t]
\centering
\includegraphics[width=\linewidth]{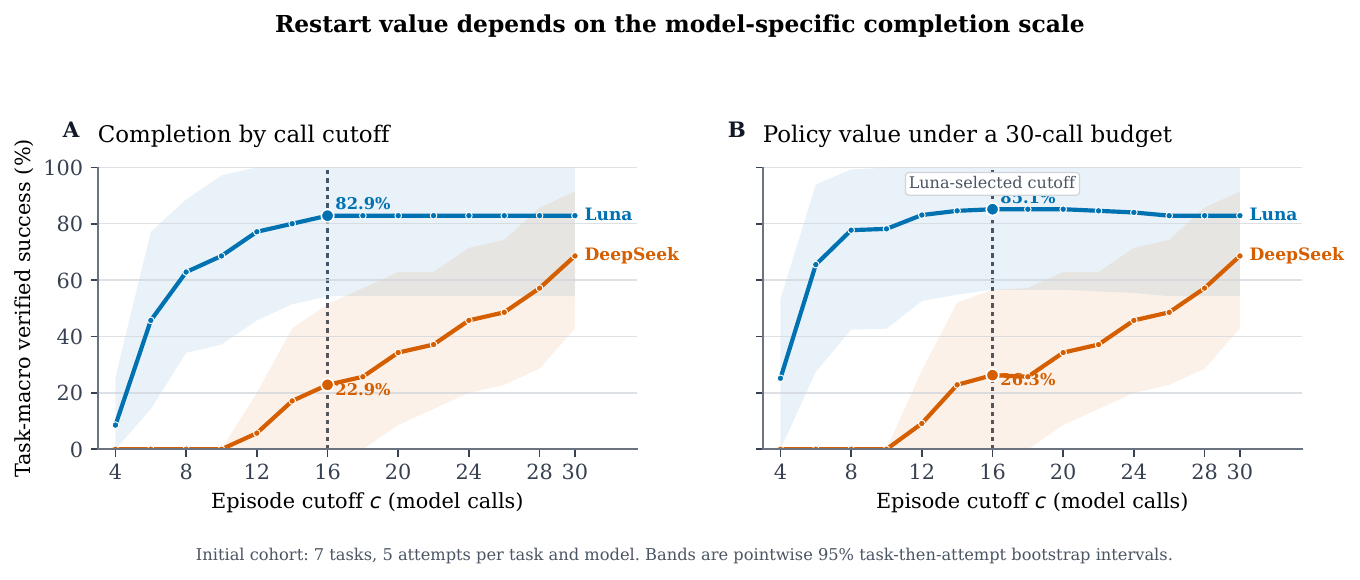}
\caption{Model-specific completion scale and restart-policy value on the initial seven tasks. Panel A reports the task-macro probability that one attempt has verified by call $c$. Panel B reports the exact plug-in success of repeated cutoff $c$ under 30 total calls. Shading gives pointwise 95\% hierarchical bootstrap intervals. Enlarged points mark the Luna-selected cutoff of 16 calls.}
\label{fig:cutoffs}
\end{figure}

The best Luna cutoff spans calls 16 through 20 and reaches $85.14\%$, compared with $82.86\%$ for one 30-call attempt. At cutoff 16, the estimated change is $+2.29$ points with a task-hierarchical interval of $[0.00,6.86]$ and bootstrap positive fraction $0.504$. This is weak in-sample evidence.

Applying the same 16-call scale to DeepSeek changes success from $68.57\%$ to $26.29\%$, a $-42.29$ point effect with interval $[-69.14,-17.71]$. DeepSeek's best cutoff is the full 30 calls. The direction of transfer in the other direction is conservative: applying 30 calls to Luna retains $82.86\%$ and gives up only the small in-sample oracle difference.

Leave-one-task-out selection removes the apparent Luna benefit. It yields $75.54\%$ against $82.86\%$ for the 30-call baseline, a $-7.31$ point difference. The decisive held-out task is 24443: the other six tasks select cutoff 10, while all five successful runs of task 24443 require 8 to 16 calls and repeated cutoff 10 has only $48.8\%$ plug-in success. DeepSeek selects 30 calls in every fold and ties its baseline.

Finite geometric schedules reach at most approximately $84.59\%$ on Luna and $22.88\%$ on DeepSeek under the preregistered family. The tested Luby schedules range from $45.72\%$ to $67.40\%$ on Luna and have zero success on DeepSeek because none allocates an episode long enough for an observed completion. This is the scale sensitivity of \cref{prop:no-dominance}, not a contradiction of asymptotic restart theory.

\finding{The same tasks and harness do not determine a portable restart scale. Changing only the model moves useful completion mass beyond the selected cutoff and turns a small apparent gain into a large loss.}

\subsection{The in-sample gain does not survive prospective transfer}

The two prospective cohorts add 60 Luna attempts. All 35 attempts in the mechanically selected high-ID cohort pass by call 15, so both policies score $100\%$. In the difficulty-enriched cohort, only one of 25 attempts passes, at call 23. Fixed 16 therefore scores zero while fixed 30 scores $4\%$. Results are summarized in \cref{tab:prospective}.

\begin{table}[t]
\centering
\caption{Transfer of the Luna-selected 16-call cutoff. The interval resamples tasks and attempts.}
\label{tab:prospective}
\small
\begin{tabular}{@{}lrrrc@{}}
\toprule
\textbf{Cohort} & \textbf{Tasks / attempts} & \textbf{Fixed 16} & \textbf{Fixed 30} & \textbf{Difference and 95\% interval} \\
\midrule
Initial & 7 / 35 & 85.14\% & 82.86\% & $+2.29$ [$0.00,6.86$] \\
Prospective high ID & 7 / 35 & 100.00\% & 100.00\% & $0.00$ [$0.00,0.00$] \\
Difficulty enriched & 5 / 25 & 0.00\% & 4.00\% & $-4.00$ [$-16.00,0.00$] \\
\midrule
Pooled & 19 / 95 & 68.21\% & 68.42\% & $-0.21$ [$-4.21,2.53$] \\
\bottomrule
\end{tabular}
\end{table}

Across all 19 Luna tasks, fixed 16 scores $68.21\%$ and fixed 30 scores $68.42\%$. The difference is $-0.21$ points with interval $[-4.21,2.53]$ and bootstrap positive fraction $0.317$. The pooled in-sample oracle is cutoff 24 at $68.84\%$, only $0.42$ points above fixed 30. Leave-one-task-out cutoff selection scores $67.37\%$, a $-1.05$ point difference from fixed 30.

Under total budget 60, a conservative policy of 15, 15, then a reserved 30-call attempt scores $69.43\%$. Two 30-call attempts score $70.11\%$, for a difference of $-0.67$ points with interval $[-2.53,0.84]$. Reserving a full episode avoids catastrophic truncation, but the short probes displace another useful long attempt.

\finding{The positive cutoff result does not replicate prospectively. Easy tasks create no headroom, while the only stress-cohort success occurs after the selected cutoff.}

\subsection{Policy differences do not certify individual restart decisions}
\label{sec:audit-results}

The retrospective audit separates where a decision occurs, which action has higher estimated value, and whether the estimate resolves its sign. Of 156 task-time cells, 103 contain at least one live attempt; the other 53 are never reached. None of the 103 live comparisons is certified by the simultaneous interval rule. This is a finite-sample limitation of the specified conservative procedure, not evidence that all actions have equal value.

At the previously selected cutoff of 16, only 13 task-model cells remain live. \Cref{tab:audit} summarizes the decision opportunities; \cref{fig:decision-evidence} displays conditional advantages and uncertainty. Five initial Luna tasks have completed in every attempt. Task 23413 has one survivor among five. For that survivor, estimated continuation success is zero and fresh success within the remaining 14 calls is $4/5$. Its task-macro contribution is therefore
\begin{equation}
  \frac{1}{7}\cdot\frac{1}{5}\cdot\frac{4}{5}=0.022857.
\end{equation}
The remaining live Luna task, 23950, has no observed success in either arm. Thus the entire $2.29$-point initial gain is supplied by one surviving failed attempt and the early successes of other draws on that same issue. The headline effect is not evidence of a broad stagnant-trajectory phenotype.

\begin{table}[t]
\centering
\caption{Decision opportunities at call 16 with total budget 30. Changes are task-macro percentage points relative to continuation. The oracle knows future continuation failure and uses the specified independent clean restart. Its values are plug-in diagnostics, not statistically validated upper bounds. No live cell receives a sign certificate.}
\label{tab:audit}
\small
\begin{tabularx}{\linewidth}{@{}Xrrrr@{}}
\toprule
\textbf{Cohort} & \textbf{Attempts} & \textbf{Live} & \textbf{Restart all} & \textbf{Perfect alarm} \\
\midrule
Initial Luna & 35 & 6 & $+2.29$ & $+2.29$ \\
Initial DeepSeek & 35 & 27 & $-42.29$ & $0.00$ \\
Prospective high ID & 35 & 0 & $0.00$ & $0.00$ \\
Prospective stress & 25 & 25 & $-4.00$ & $0.00$ \\
\bottomrule
\end{tabularx}
\end{table}

All six live DeepSeek task cells favor continuation in the plug-in estimates. The prospective high-ID cohort supplies no live decision at all: both policies tie because completion precedes the alarm time. The stress cohort supplies many failures, but no observed fresh completion within 14 calls. This is a failure-rich evaluation set with little measured recovery opportunity under the chosen action. A detector cannot be usefully assessed for completion gains merely by selecting trajectories likely to fail.

The perfect-alarm calculation in \cref{cor:opportunity} adds no estimated completion gain beyond continuation in either prospective cohort or the initial DeepSeek cohort at this decision time. This statement is narrowly conditional on the empirical distributions, call 16, and one independent clean restart. It does not bound a richer recovery action, a different decision time, or the true opportunity concealed by limited sampling. Even oracle diagnostics require uncertainty assessment.

\begin{figure}[t]
\centering
\includegraphics[width=\linewidth]{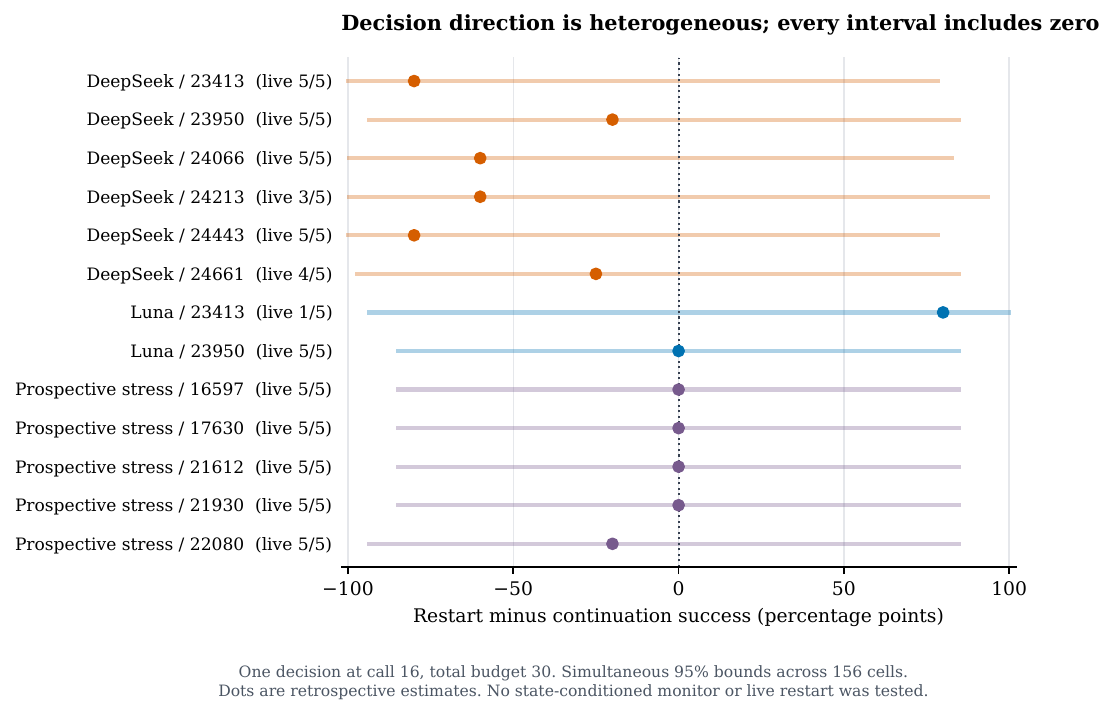}
\caption{Conditional restart advantage in the 13 task-model cells with a live attempt at call 16. Dots are plug-in estimates. Lines are simultaneous difference bounds formed from binomial intervals across all 156 audited comparisons. Positive values favor restarting. Every interval includes zero despite substantial heterogeneity in point estimates.}
\label{fig:decision-evidence}
\end{figure}

The lack of individual certificates is compatible with the large cross-model policy loss. The earlier bootstrap measures an aggregate policy contrast across sampled tasks; the audit asks for the sign of each conditional action comparison with simultaneous protection. Five attempts per task, sometimes reduced to one live continuation, cannot provide comparable resolution. For example, \cref{eq:resolution-cost} requires 90 zero-success observations to reduce the specified upper bound to $10\%$ when $M=156$, versus 42 for one decision. These are procedure-specific evidence costs, not lower bounds on every possible learning method.

\finding{Failure prevalence, decision opportunity, and statistical support are different quantities. The same audit set can contain many failures, negligible measured recovery opportunity, and insufficient evidence to certify any individual action.}

\subsection{Additional horizon reveals protocol collapse, not late success}

The horizon extension tests four tasks that were previously $0/5$ at 30 calls. Among 20 fresh 60-call attempts, one passes at call 17. No completion first appears during calls 31 through 60. An unclustered one-sided $95\%$ binomial upper bound is $13.9\%$. This attempt-level sensitivity calculation does not provide population coverage over heterogeneous tasks; rare late completion and repository-level variation remain unresolved.

Seventeen attempts end after three consecutive invalid tool responses and two exhaust the horizon. On the same four tasks at horizon 30, only one of 20 attempts had ended in repeated-format failure. The attempt-level Fisher exact comparison of 17/20 with 1/20 gives $p=4.06\times10^{-7}$; it is descriptive because attempts are clustered within four tasks. Fourteen of the 17 format failures occur after call 30.

\begin{table}[t]
\centering
\caption{Fresh 60-call outcomes on tasks previously unsolved in five 30-call attempts.}
\label{tab:horizon}
\small
\begin{tabular}{@{}lrrrrr@{}}
\toprule
\textbf{Task} & \textbf{Pass} & \textbf{First pass} & \textbf{Format fail} & \textbf{Limit} & \textbf{Mean cost} \\
\midrule
16597 & 0/5 & none & 4 & 1 & \$0.05716 \\
17630 & 0/5 & none & 5 & 0 & \$0.03577 \\
21612 & 1/5 & call 17 & 4 & 0 & \$0.03126 \\
21930 & 0/5 & none & 4 & 1 & \$0.04996 \\
\midrule
Total & 1/20 & no pass after 30 & 17 & 2 & \\
\bottomrule
\end{tabular}
\end{table}

At total budget 60, repeated fixed 30 scores $9\%$ by empirical replay, compared with $5\%$ for one 60-call attempt. A policy that allows the full 60 calls but immediately recycles budget after explicit terminal failure also scores $9\%$. Hence the apparent restart benefit comes entirely from reusing calls after an observable protocol terminal event. Arbitrarily censoring a live trajectory adds nothing.

\finding{More nominal horizon increases exposure to protocol failure but reveals no observed late completion. Recycling budget after an explicit failure is distinct from declaring a still-running trajectory stagnant.}

\subsection{Fresh context changes the terminal label, not success}

The context intervention fires in 18 of 20 treated attempts. Baseline solves one of 20; treatment solves none. The task-clustered bootstrap difference is $-5$ points with empirical range $[-20,0]$. Since treatment has no successes, this range is boundary-degenerate and cannot represent a positive treatment event. An unclustered Newcombe sensitivity interval is $[-23.6,15.2]$ points.

The treatment does eliminate the final \texttt{RepeatedFormatError} label. Every triggered attempt instead consumes its remaining calls, and all 20 treated attempts end at \texttt{LimitsExceeded}. Mean dollar cost rises from \$0.03786 to \$0.05934, a $56.7\%$ increase. Mean wall time rises from 309.5 to 512.8 seconds, a $65.7\%$ increase. \Cref{fig:rejuvenation} displays the failure-mode substitution.

\begin{figure}[t]
\centering
\includegraphics[width=\linewidth]{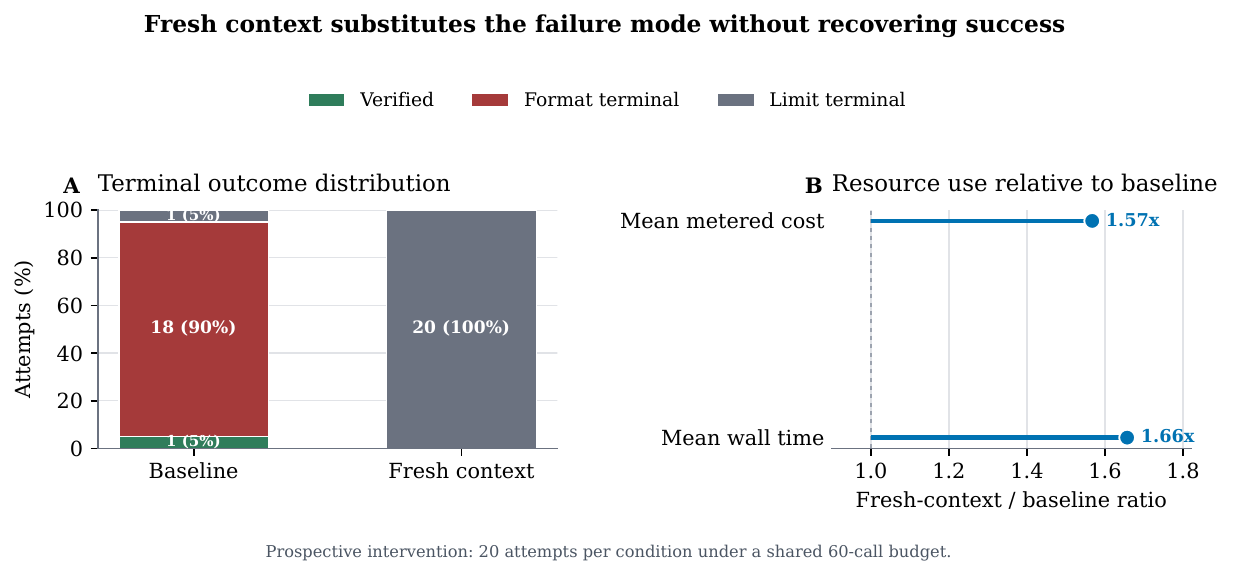}
\caption{Context rejuvenation changes the observed failure mode without improving task success. Panel A shows terminal outcomes, with counts and within-condition percentages. Panel B shows the corresponding increase in mean metered cost and wall time relative to baseline.}
\label{fig:rejuvenation}
\end{figure}

\begin{table}[t]
\centering
\caption{Context-rejuvenation outcomes under the shared 60-call budget.}
\label{tab:rejuvenation}
\small
\begin{tabular}{@{}lrrr@{}}
\toprule
\textbf{Outcome} & \textbf{Baseline} & \textbf{Rejuvenation} & \textbf{Treatment difference} \\
\midrule
Verified success & 5.0\% & 0.0\% & $-5.0$ points \\
Repeated-format terminal & 90.0\% & 0.0\% & $-90.0$ points \\
Limit terminal & 5.0\% & 100.0\% & $+95.0$ points \\
Mean cost & \$0.03786 & \$0.05934 & +\$0.02148 \\
Mean wall time & 309.5 s & 512.8 s & +203.3 s \\
\bottomrule
\end{tabular}
\end{table}

The frozen expansion criterion is not met. Fresh inference context removes the terminal interface symptom but preserves the workspace, including potentially wrong edits. The tested reset does not support conversational contamination as a sufficient explanation for recoverable failure. It does not isolate which retained edits or other mechanisms caused the remaining failures.

\finding{An intervention can improve a harness diagnostic by 90 points while leaving success no better and making cost materially worse. Failure-mode movement is not task progress.}

\subsection{Failed branches are diverse but not compositional}

The 40-run context experiment yields 38 admissible failed sources and 37 unique patch contents. Of 162 unordered source pairs, 51 form clean, novel three-way composites. The frozen verifier-blind ranking selects 19. None passes. Task 17630 produces no clean novel composite because branches edit the same matrix-code region incompatibly.

\begin{table}[t]
\centering
\caption{Verifier-blind patch composition from failed terminal branches.}
\label{tab:patch-algebra}
\small
\begin{tabular}{@{}lrrrrrr@{}}
\toprule
\textbf{Task} & \textbf{Sources} & \textbf{Unique} & \textbf{Pairs} & \textbf{Clean novel} & \textbf{Evaluated} & \textbf{Pass} \\
\midrule
16597 & 10 & 9 & 45 & 35 & 8 & 0 \\
17630 & 10 & 10 & 45 & 0 & 0 & 0 \\
21612 & 9 & 9 & 36 & 13 & 8 & 0 \\
21930 & 9 & 9 & 36 & 3 & 3 & 0 \\
\midrule
Total & 38 & 37 & 162 & 51 & 19 & 0 \\
\bottomrule
\end{tabular}
\end{table}

One bounded semantic resolver is then run per task. It receives every unique failed patch as an explicitly failed hypothesis and starts from a clean workspace. All four runs use the 12-call cap and fail verification. The patches are not copies: their maximum character-level diff similarity to any supplied parent is 0.563, 0.700, 0.645, and 0.802. The task-level relevant tests pass in counts 76/79, 21/22, 99/100, and 47/52, but the protocol forbids extending the near miss after outcome inspection.

The historical clean-context cohort also has no completion within 12 calls on these tasks, so the resolver screen provides no evidence of improvement and is not treated as a randomized effect estimate. The null result has a more focused interpretation: syntactic novelty and access to multiple failed views are insufficient evidence of independent corrective information.

\finding{Failed branches occupy different textual locations without supplying a verified closure operation. Clean union and shared-model synthesis both preserve diversity without establishing correctness.}

\begin{figure}[t]
\centering
\includegraphics[width=\linewidth]{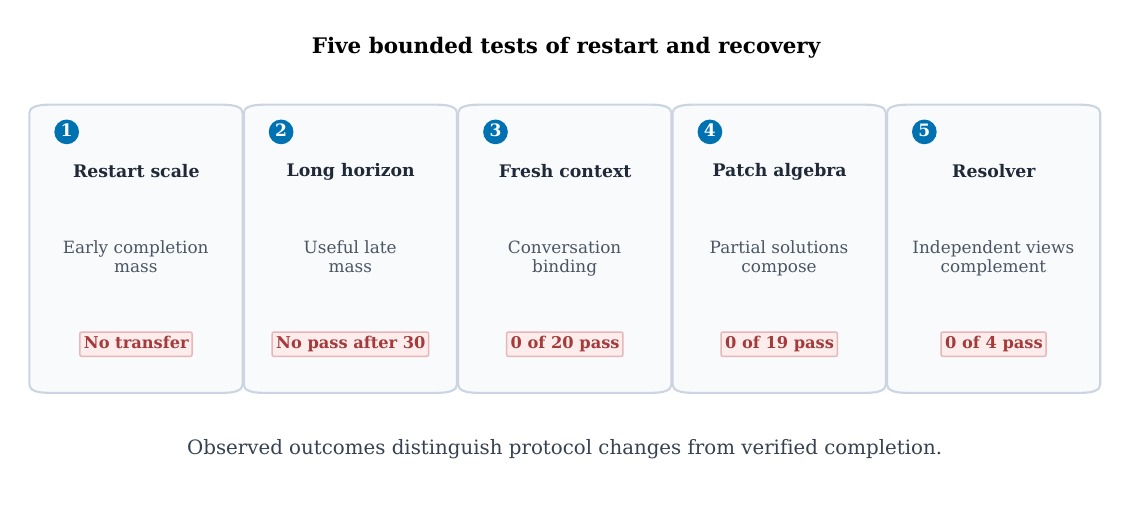}
\caption{Five bounded tests of restart and recovery. Outcomes identify limitations of the tested interventions, not impossibility results for richer controllers. The task populations and uncertainty differ across stages.}
\label{fig:falsification-ladder}
\end{figure}

\section{Discussion}
\label{sec:discussion}

\subsection{Contribution and scope}

The contribution is an operational account of when progress information can improve a specified agent-control decision. Classical restart optimization already allows informative observations, and contemporary learned monitors demonstrate successful termination and recovery. Our results do not overturn those findings. They connect failure discrimination to residual-budget restartability, distinguish decision-preserving compression from trace reconstruction, and expose the sampling support of per-task action comparisons.

Three levels of evidence must remain separate. The mathematical identities hold under stated assumptions and specialize established decision theory. The retrospective audit diagnoses the collected traces and supplies a reusable calculation, but does not demonstrate a deployed policy gain. The prospective intervention screens establish outcomes of particular controllers on selected tasks, not universal limits on monitoring or recovery. Combining these levels makes the paper useful only if their inferential boundaries are preserved.

\subsection{Implications for evaluation and deployment}

\blockparagraph{Measure opportunity before monitor quality.}
A failure detector should be evaluated where an alternative action can plausibly recover value. A ceiling cohort offers no live decisions; a hard cohort can offer many failures but little residual-budget restartability. The perfect-alarm reference in \cref{cor:opportunity} is a diagnostic for this distinction. A small plug-in ceiling warrants additional evidence or a different intervention, not a conclusion that every progress signal is useless.

\blockparagraph{Match the signal to the action.}
A clean restart, conversational reset with mandatory workspace retention, and optional reuse of previous edits define different potential outcomes. Their action values cannot be interchanged. FailFast-RestartSmart~\citep{wang2026failfast}, for example, evaluates a learned monitor with a recovery mechanism not tested here. Its positive results and our negative timeout screens can both hold. A meaningful comparison must fix reset semantics, residual budget, verifier, and monitoring cost.

\blockparagraph{Preserve decision distinctions, not every token.}
The information-value identity gives a precise connection to context compression. A representation can discard large parts of a trace without losing control value if the discarded distinctions do not reverse the preferred action. Conversely, faithful reconstruction can preserve many irrelevant details. This is a criterion for future evaluation, not evidence that our study has built a superior compressor.

\blockparagraph{Make abstention explicit.}
The released audit returns an action only when its specified intervals separate. Otherwise it reports insufficient evidence. This is useful as a calibration check before enabling an automatic restart rule; it is not a claim that continuation is optimal whenever the audit abstains. The present five-run cells are too sparse for this conservative task-specific rule. Hierarchical pooling could improve efficiency, but would introduce assumptions that require held-out validation.

\blockparagraph{Separate protocol repair from task recovery.}
Observable termination is not a live stagnation decision. Recycling its unused calls recovers budget in our replay. Replacing its final error label without changing verified completion is a different result. The context intervention illustrates why dashboards should report final task outcomes and cost alongside protocol diagnostics. Likewise, failed-patch diversity alone does not establish that a particular composition or synthesis operator can recover a solution. We do not estimate error correlation or identify the causal source of the unsuccessful combinations.

\subsection{Limitations}
\label{sec:limitations}

All 19 tasks are selected SymPy issues. The cohort construction deliberately mixes a pilot, a mechanical high-ID selection, and difficulty enrichment; pooled rates do not estimate performance over SWE-bench Verified or agent workloads generally. Later interventions use one model and one compact harness. Two models on seven shared tasks demonstrate heterogeneity, not broad model coverage.

Five attempts per task leave substantial uncertainty, particularly after conditioning on live survival. The zero-success treatment and reuse screens do not exclude modest benefits in other populations. The unclustered late-tail calculation is a sensitivity result, not a task-population confidence statement. The retrospective decision family was chosen after the original experiments and is not a prospective validation of the certification rule.

The analysis assumes exchangeable independent provider draws within each fixed task and runtime. Shared routing, temporal drift, or correlated restarts can invalidate that approximation. The binomial certificate does not transfer across arbitrary changes in the task distribution, prompt, harness, provider, or recovery action. It also provides no protection for uncorrected adaptive monitoring.

The binary verifier measures a finite fail-to-pass specification, not every aspect of issue correctness. Base-fail/reference-pass gating and successful-patch audits reduce implementation threats but cannot eliminate specification gaps. Per-action external evaluation supplies completion labels unavailable to a typical production controller. Its computational overhead is not part of the model-call budget; reported model costs are not end-to-end deployment costs.

Finally, we have no matched prefix-conditioned continuation and restart trials for a learned signal. The empirical audit uses task identity and elapsed calls, not rich trace features. It cannot measure the decision value of embeddings, learned monitors, or semantic summaries. The information and opportunity identities describe what such experiments should estimate; they do not substitute for them.

\subsection{A focused next test}

The next decisive experiment should evaluate one existing monitor under the same recovery semantics as a simple baseline, rather than add another timeout family. First, use independent calibration data to identify tasks with both live failures and nontrivial residual-budget recovery. Freeze the monitor, threshold, budget, and clean-reset or edit-reuse rule. On held-out tasks, compare continuation, a calibrated fixed-time rule, and the monitor-triggered intervention, charging monitoring and recovery overhead to the same budget. Report opportunity, calibration, action-value regret, success, and cost separately.

A second model or repository is necessary to test transfer; more draws on the existing four hard tasks cannot establish it. Such a study would convert the current measurement contribution into evidence for a deployable control method. It is deliberately distinct from the present offline audit and should not be described as completed.

\section{Conclusion}

Failure prediction is useful for agent control only through the value of the action it triggers. For an independent clean restart, that value depends jointly on alarm errors and success within the residual budget. Classical information comparison yields a decision-oriented criterion for retaining or compressing progress observations, while exact binomial intervals expose when a per-task action preference lacks statistical support.

Across 194 model executions and 19 verifier-only composites, simple restart and recovery interventions show strong runtime dependence and no prospective replication of the initial cutoff gain. A new audit of 130 existing traces finds 103 live task-time comparisons but no sign certificate under its simultaneous rule. These findings motivate an evaluation standard: declare the intervention, measure recovery opportunity, account for monitoring cost, and report uncertainty at the level where decisions are made. \restartlab{} provides the accompanying trace analysis and audit implementation; broader prospective validation remains necessary.

\bibliographystyle{restartlab-hyphen}
\bibliography{references}

\appendix
\section{Notation}
\label{app:notation}

\begin{table}[h]
\centering
\caption{Principal notation.}
\label{tab:notation}
\small
\begin{tabularx}{\linewidth}{@{}p{0.20\linewidth}X@{}}
\toprule
\textbf{Symbol} & \textbf{Meaning} \\
\midrule
$X_t$ & Latent agent and environment state after call $t$ \\
$O_t$, $\Fobs_t$ & Observation at call $t$ and the observable history filtration \\
$T$ & First call whose resulting state passes the external verifier \\
$E$, $H$ & Explicit protocol-terminal time and configured attempt horizon \\
$B$, $b$ & Total policy budget and residual budget \\
$F(t)$, $\lambda_t$ & Completion-time CDF and discrete completion hazard \\
$Q_a(h,b)$ & Conditional utility of action $a$ after history $h$ with residual budget $b$ \\
$\Delta(h,b)$ & Continuation advantage over restart \\
$G_t$ & Observable candidate progress representation \\
$m_{\mathrm{eff}}$ & Effective number of independent views in a correlated ensemble \\
\bottomrule
\end{tabularx}
\end{table}

\section{Task Cohorts and Frozen Decisions}
\label{app:cohorts}

\begin{table}[h]
\centering
\caption{Task IDs by cohort. IDs omit the common \texttt{sympy\_\_sympy-} prefix.}
\label{tab:task-list}
\small
\begin{tabularx}{\linewidth}{@{}lX@{}}
\toprule
\textbf{Cohort} & \textbf{Task IDs} \\
\midrule
Initial & 23413, 23824, 23950, 24066, 24213, 24443, 24661 \\
Prospective high ID & 22456, 22714, 22914, 23262, 23534, 24539, 24562 \\
Difficulty enriched & 16597, 17630, 21612, 21930, 22080 \\
Horizon 60 & 16597, 17630, 21612, 21930 \\
Context A/B & 16597, 17630, 21612, 21930 \\
Patch and resolver & 16597, 17630, 21612, 21930 \\
\bottomrule
\end{tabularx}
\end{table}

The task sets overlap across stages by design. Counts of 19 unique tasks and 194 model attempts deduplicate task identity but not fresh executions. The total consists of 70 initial attempts, 60 prospective attempts, 20 horizon-extension attempts, 40 context A/B attempts, and four semantic-resolver attempts. The patch-algebra stage adds 19 external verifier calls on composed candidates but no model attempt.

\section{Policy Definitions}
\label{app:policies}

\blockparagraph{Single.}
Run one attempt up to budget $B$, stopping at first external pass or explicit terminal failure.

\blockparagraph{Fixed $c$.}
Run independent attempts with per-episode cap $c$ until the total budget is spent. An explicit terminal failure releases unused episode budget.

\blockparagraph{Geometric.}
Allocate episode caps in a frozen geometric sequence, truncated by the remaining total budget.

\blockparagraph{Luby.}
Allocate episode caps according to a scaled prefix of the universal sequence in \citet{luby1993optimal}, truncated by the remaining total budget.

\blockparagraph{Horizon insurance.}
At budget 60, allocate two 15-call probes followed by one reserved 30-call episode. Compare with two 30-call episodes.

The fixed-policy plug-in estimate treats the five empirical attempts for each task as an IID distribution. For a sequence of episode caps $(c_1,\ldots,c_k)$ and empirical CDF $\widehat F_j$, the task success probability under independent resampling is
\begin{equation}
  1-\prod_{i=1}^{k}\left(1-\widehat F_j(c_i)\right).
\end{equation}
Monte Carlo replay is used in the implementation when terminal failures create episode-dependent residual budgets.

\section{Additional Operational Results}
\label{app:operations}

\begin{table}[h]
\centering
\caption{Collection cost and accumulated wall time. Wall time sums parallelizable attempts and is not elapsed study duration.}
\label{tab:costs}
\small
\begin{tabular}{@{}lrrr@{}}
\toprule
\textbf{Stage} & \textbf{Attempts} & \textbf{Metered cost} & \textbf{Wall hours} \\
\midrule
Initial Luna & 35 & \$0.38327 & 0.771 \\
Initial DeepSeek & 35 & \$0.38023 & 3.644 \\
Prospective Luna & 60 & \$1.14379 & 2.596 \\
Horizon 60 & 20 & \$0.87077 & 2.207 \\
Context A/B & 40 & \$1.94403 & 4.568 \\
Semantic resolver & 4 & \$0.05132 & 0.145 \\
\midrule
Total & 194 & \$4.77341 & 13.931 \\
\bottomrule
\end{tabular}
\end{table}

Costs are provider-metered values recorded in the traces. They are not estimates of deployment expense. The large cost and latency change in context rejuvenation arises because treatment consumes all 60 calls after replacing an early terminal format failure with a fresh episode.

\section{Artifact Manifest}
\label{app:artifacts}

The accompanying analysis package contains the following sanitized artifacts. They exclude prompts, trajectories, local paths, hidden verifier output, reference patches, and credentials.

\begin{table}[h]
\centering
\caption{Accompanying analysis artifacts and their role in the study.}
\label{tab:artifacts}
\small
\begin{tabularx}{\linewidth}{@{}p{0.31\linewidth}p{0.25\linewidth}X@{}}
\toprule
\textbf{Artifact cohort} & \textbf{Scope} & \textbf{Primary use} \\
\midrule
Initial paired-model cohort & 70 attempts, 7 tasks & Completion curves and restart replay \\
Prospective transfer cohorts & 60 attempts, 12 tasks & Out-of-sample cutoff transfer \\
Horizon extension & 20 attempts, 4 tasks & Late completion and protocol termination \\
Context intervention trial & 40 attempts, 4 tasks & Terminal outcomes, cost, and latency \\
Patch-composition screen & 19 verifier-only candidates & Failed-branch compositionality \\
Semantic resolver screen & 4 attempts, 4 tasks & Semantic aggregation \\
Decision audit & 156 derived cells from 130 existing attempts & Action values, eligibility, and simultaneous bounds \\
\bottomrule
\end{tabularx}
\end{table}

The package also contains stage protocols, the frozen composition plan, agent adapters, replay policies, analysis code, command-line tools, and tests. The sanitized artifacts reproduce the reported trace analyses and figures offline; new agent collection requires a model credential. Derived audit cells reuse earlier attempts and do not increase the count of model executions.

\section{Threats to Interpretation}
\label{app:threats}

\blockparagraph{Task selection.}
The first seven tasks were an operational pilot, the prospective high-ID cohort used a mechanical order, and the stress cohort deliberately enriched difficulty. Pooled rates are descriptive for these tasks, not prevalence estimates for SWE-bench Verified.

\blockparagraph{Verifier granularity.}
Verification after every environment action yields precise first-pass timing, but the fail-to-pass suite is still a finite specification. We report verifier success rather than semantic issue resolution and retain the standard benchmark terminology only when naming the dataset.

\blockparagraph{Adaptive sequence.}
Later stages were chosen after earlier results. Each stage has its own frozen protocol, but cross-stage narratives can still invite hindsight. The ladder table therefore states the question and comparison at each stage, and null follow-ups are retained rather than omitted.

\blockparagraph{Multiplicity.}
The study evaluates several policies and mechanisms. The stage-level bootstrap intervals are not adjusted for simultaneous coverage. Only the decision-audit intervals use the explicit family allocation in \cref{sec:certificate}. We avoid claiming significance for exploratory maxima, identify in-sample oracles as nondeployable, and give priority to prospective contrasts.

\blockparagraph{Model identity.}
Hosted behavior can change under a stable model identifier. Requested model ID and collection date therefore define, but do not fully identify, the runtime used for calibration.

\section{Failure-Conditioned Branches: Supporting Examples}
\label{app:ensemble}

Branching is often proposed when continuation is uncertain. Its benefit depends on the epistemic relation among branches, not their textual distance. Conditioning all inputs to an ensemble on terminal failure can also reverse component reliability.

\begin{lemma}[Failure-conditioned component inversion]
\label{lem:selection}
Let $Y\sim\mathrm{Bernoulli}(q)$ indicate that one component of a patch is correct, with $q>1/2$. Let $Z\sim\mathrm{Bernoulli}(r)$ indicate that the remaining requirements are correct, independent of $Y$, and let full success be $S=YZ$. Then
\begin{equation}
  \Prob(Y=1\mid S=0)=\frac{q(1-r)}{1-qr}<\frac12
\end{equation}
whenever $r>(2q-1)/q$.
\end{lemma}

\begin{proof}
Bayes' rule gives the displayed conditional probability. Comparing it with $1/2$ yields $2q(1-r)<1-qr$, which is equivalent to $r>(2q-1)/q$.
\end{proof}

Thus a component that is reliable in the unconditional branch population can become minority-reliable after selecting only failed full patches. The lemma is an existence result, not an estimate of our patch population.

Correlation supplies a second limit. Let $Y_1,\ldots,Y_m$ be exchangeable correctness indicators with mean $q$ and common pairwise correlation $\rho$. Then
\begin{equation}
  \operatorname{Var}\!\left(\frac{1}{m}\sum_{i=1}^{m}Y_i\right)
  =\frac{q(1-q)}{m}\left[1+(m-1)\rho\right].
  \label{eq:correlated-ensemble}
\end{equation}
The variance is that of only
\begin{equation}
  m_{\mathrm{eff}}=\frac{m}{1+(m-1)\rho}
\end{equation}
independent views. If $\rho=1$, increasing the number of branches does not reduce uncertainty at all. This standard ensemble effect~\citep{dietterich2000ensemble,brown2005diversity} motivates our distinction between \emph{behavioral diversity}, such as different files or diffs, and \emph{epistemic independence}, meaning error diversity sufficient for aggregation. Branches from one model and harness can satisfy the former while failing the latter.

\end{document}